\documentclass[11pt]{article}

\usepackage{fullpage}
\usepackage{amsmath}
\usepackage{amssymb}
\usepackage{xspace}
\usepackage{epsfig}
\usepackage{theorem}
\usepackage{ifthen}
\usepackage{hyperref}
\usepackage{algorithm}
\usepackage{algorithmic}

\newcommand{\algo}{{\cal A}}
\newcommand{\ouralg}{\bar{{\cal A}}}
\newcommand{\be}{\begin{equation}}
\newcommand{\ee}{\end{equation}}
\newcommand{\comment}[1]{}
\newcommand{\secref}[1]{Section \ref{#1}}
\newcommand{\ignore}[1]{}

\def\sse{\subseteq}
\def\sm{\setminus}

\def\F2{\mathcal{F}^{(2)}}
\def\C{\mathcal{C}}

\def\F{\mathcal{F}}
\def\G{\mathcal{G}}
\def\q{\mathbf{q}}

\def\Y{Y}
\def\y{y}
\def\X{X}
\def\Yhat{\widehat{Y}}
\def\q{q}

\newcommand{\prob}[2][]{\text{\bf Pr}\ifthenelse{\not\equal{}{#1}}{_{#1}}{}\!\left[#2\right]}
\newcommand{\expect}[2][]{\text{\bf E}\ifthenelse{\not\equal{}{#1}}{_{#1}}{}\!\left[#2\right]}

\newcommand{\val}{v}

\def\v2{\val^{(2)}}

\newtheorem{theorem}	 			{Theorem}[section]
\newtheorem{lemma}		[theorem]	{Lemma}

{\theorembodyfont{\rmfamily} \newtheorem{definition}
[theorem]	{Definition}}
{\theorembodyfont{\rmfamily} \newtheorem{remark}		[theorem]
{Remark}}
{\theorembodyfont{\rmfamily} }

\newenvironment{proof}{\noindent {\em {Proof:}}}{$\blacksquare$\vskip \belowdisplayskip}

\title{Tight Error Bounds for Structured Prediction}

\author{Amir Globerson
\and
Tim Roughgarden
\and
David Sontag
\and
Cafer Yildirim}

\begin{document}

\maketitle
\thispagestyle{empty}

\begin{abstract}
Structured prediction tasks in machine learning involve the
simultaneous prediction of multiple labels. This is typically done by
maximizing a score function on the space of labels, which decomposes
as a sum of pairwise elements, each depending on two specific labels.
Intuitively, the more pairwise terms are used, the better the expected
accuracy. 
However, there is currently no theoretical account of this intuition.
%Furthermore, it is important to also consider the
%computational hardness of the prediction algorithm and its interaction
%with the expected accuracy. 
This paper takes a significant step in this direction. 
%Our work is 

We formulate the problem as classifying the vertices of a
known graph $G=(V,E)$, where the vertices and edges of the graph are
labelled and correlate semi-randomly with the ground truth.
We show that the prospects for achieving low expected Hamming error
depend on the structure of the graph $G$ in interesting ways.  
For example, if $G$ is a very poor expander, like a path, then
large expected Hamming error is inevitable.
Our main positive result shows that, 
for a wide class of graphs including 2D grid graphs common in machine
vision applications, there is a polynomial-time algorithm with small
and information-theoretically near-optimal expected error.
%We analyze the expected Hamming error in structured
%prediction settings, focusing on the case of 2D grid graphs which are
%common in machine vision applications. 
%Our main results provide tight
%upper and lower bounds on this error, as well as a polynomial-time
%algorithm which achieves the bound. 
%previously believed  
%to be NP-hard. %, as found throughout computer vision and natural language  processing
Our results provide a first step toward a theoretical justification   
for the empirical success of the efficient approximate inference
algorithms that are used for structured prediction in models 
where exact inference is intractable.

%(``+'' suggesting that a pair of objects should receive the same
%label, ``-''
%for different labels).
\end{abstract}

\newpage

\section{Introduction}\label{sec:intro}

An increasing number of problems in machine learning are
being solved using structured prediction \cite{Collins02, Lafferty01conditional,
  Taskar03}. Examples of structured prediction include 
dependency parsing for natural language processing, 
%\cite{mcdonald2005online},
part-of-speech tagging 
%\cite{Lafferty01conditional},
named entity recognition, 
% \cite{mccallum2003early},
and protein folding. % \cite{bakir2007predicting}. 
In this setting, the input $X$ is some observation (e.g., an image, a sentence) and the
output is a set of labels %$\Yhat=\Yhat_1,\ldots,\Yhat_N$ 
(e.g., whether
each pixel in the image is foreground or background, or the parse tree
for the sentence). The advantage of performing structured prediction
is that one can specify features that encourage sets of labels to take
some value (e.g., a feature that encourages two neighboring pixels to
take different foreground/background states whenever there is a big
difference in their colors). The feature vector can then be used
within an exponential family distribution over the space of labels,
conditioned on the input. The parameters are learned using maximum
likelihood estimation (as with conditional random fields \cite{Lafferty01conditional}) or using
structured SVMs \cite{Altun03,Taskar03}.
 
In the applications above, performance is typically quantified as the discrepancy between the correct ``ground truth'' labels $\Y$
and the predicted labels $\Yhat$. The most common performance measure,
which we study in this paper, is Hamming error, the number of
disagreements between $\Y$ and 
% Amir1: Should be careful about this common wisdom statement. In structured SVM the Hamming error is minimized so the MAP is just a choice of prediction function and isn't wrong in any clear sense.
$\Yhat$. The optimal decision strategy for minimizing Hamming error is
to use marginal inference, namely $\Yhat_i \leftarrow \arg\max_{\Y_i}
p(\Y_i \mid X)$ for each $i$, where $p$ is the true generating
distribution. However, in practice MAP inference is more often used. 
% Amir1: Again, should note that there's also the fact we do not know or wish to know the true distribution.
Namely, the assignment maximizing $p(\Y|\X)$ is returned. One
advantage of using MAP inference is computational,
as the partition function (normalization constant) no longer needs to
be estimated during training or at test time. However, in the
worst case, even MAP inference can be NP-hard, such as for binary
pairwise Markov random fields with arbitrary potential functions.
%for which MAP inference can be reduced to MaxCut.

\comment{
It is common wisdom that when quantifying error using Hamming
error, it is best to predict using marginal inference,
i.e. $\Yhat_i \leftarrow \arg\max_{\Y_i} p(\Y_i \mid X)$ using the
distribution that was learned.
% NOTE: should we make this more formal? It's only optimal when the
% model is correct.
Despite this, often prediction is performed using 
  MAP inference,  $\Yhat \leftarrow \arg\max_{\Y} p(\Y \mid X)$. One
advantage of using MAP inference is computational,
as the partition function (normalization constant) no longer needs to
be estimated during training or at test time. However, in the
worst-case, even MAP inference can be NP-hard, such as for binary
pairwise Markov random fields with arbitrary potential functions,
for which MAP inference can be reduced to MaxCut.
}

%The combinatorial optimization problem that needs to be solved for
%prediction may or may not be tractable, depending on the set $\cC$.
It is now widely understood from a practical perspective that better 
performance (measured in terms of Hamming error) can be obtained by 
using a more complex model incorporating a strong set of features than 
a simple model for which exact inference can be performed.
Despite the worst-case intractability of inference in these models, 
heuristic MAP inference algorithms often work well in practice,
 including those based on linear programming relaxations and dual
 decomposition \cite{koo2010dual,SontagEtAl_uai08}, policy-based
 search \cite{daume-searn09}, graph cuts \cite{Kolmogorov_graphcuts},
 and branch-and-bound \cite{Sun_BB_CVPR2012}. By ``work well in
 practice'', we mean that they obtain high accuracy
 predicting the true labels on test data, measured in terms of the
 actual loss function of interest such as Hamming error.

However, the theoretical understanding of the setup is fairly 
limited. For example, for many applications even the state-of-the-art
structured prediction models are unable to achieve zero
labeling error, and there is no characterization of the choice of 
feature sets and the generative settings for which high prediction 
accuracy can be expected, even ignoring computational limitations.
Moreover, the good performance of these heuristic algorithms indicates
that real-world instances are far from the theoretical
worst case, and it is a major open problem to better characterize the
complexity of inference problems to distinguish those that are in fact
easy to solve from those that are computationally
intractable. Finally, it is not well understood why MAP inference can
provide such good results for these structured prediction problems
and how much accuracy is lost relative to marginal inference.

%
%\tfrac{1}{2}$, the best possible prediction is 
%takes on an extreme value (0 or $\tfrac{1}{2}$), these 
%questions are trivial.  In the latter case, the edge labels are random 
%and outputting a random classification has optimal error 
%$\tfrac{N}{2}$, where $N = |V|$.  In the former case, 
%%%deleting all the ``-'' edges reveals the correct classification. 
%pick a 
%spanning tree $T$ of $G$, an arbitrary label for some node, and assign 
%labels to the other nodes according to the ``+/-'' labels of the edges 
%of $T$.  This procedure recovers the ground truth provided none of 
%the $N-1$ edges of $T$ are bad, showing that 
%recovery remains easy when $p = o(\tfrac{1}{N})$. 
% 

The goal of this paper is to initiate the theoretical study of
structured prediction for obtaining
small Hamming error. Such an analysis must define a generative process
for the $\X,\Y$ pairs, in order to properly define expected Hamming
error. Our model assumes that the observed $\X$ is a noisy version of
$\Y$ in the following sense: $X_i$ is a noisy version of the true
$Y_i$ and $X_{i,j}$ is a noisy version of the variable ${\cal
  I}\left[Y_i=Y_j\right]$. The resulting posterior for $\Y$ given $\X$
is then very similar to the {\em data} and {\em smoothness} terms used
for structured prediction in machine vision. Motivated by machine
vision applications we also focus on the case where $i,j$ pairs
correspond to a two dimensional grid graph \cite{vicente2008graph}. We
also provide results for classes of non-grid and non-planar graphs.

As noted earlier, prediction is often performed by taking marginals of
the posterior or its maximum. Both of these turn out to be
computationally intractable in our setting. We are thus also
interested in analyzing algorithms that {\em are} polynomial time and
have guarantees on the expected Hamming error. Our main result is that
there exists a polynomial-time algorithm that achieves the
information-theoretic lower bound on the expected Hamming error, and
is thus 
optimal (up to multiplicative constants). The algorithm is a two-step
procedure which ignores the node evidence in the first step, solving a
MaxCut problem on a grid (which can be done in polynomial time), and
in the second step uses node observations to break symmetry. 
We use combinatorial arguments to
provide a worst-case upper bound on the error of this algorithm.
% by showing that its
%predictions must satisfy certain structural properties. 
Our analysis
is validated via experimental results on 2D grid graphs.

\section{Related Work \label{ss:related}}
Our goal is to recover a set of unobserved variables $\Y$ from a set of noisy observations $\X$.
 As such it is related to various statistical recovery settings, but distinct from those in several important aspects. Below we review some of the related problems.
 
{\bf Channel Coding:} This is a classic recovery problem
(e.g., see \cite{arora2009message}) where the goal is to exactly
recover $\Y$ (i.e., with zero error). Here $\Y$ is augmented with a
set of ``error-correcting'' bits, deterministic functions of
$\Y$, and the complete set of bits is sent through a noisy
channel. In our model, $\X_{i,j}$ is a noisy version of the parity of
$Y_i$ and $Y_j$. Thus
our setting may be viewed as communication with an error correcting
code where each error-correcting bit involves two bits of the original
message $\Y$, and each $Y_i$ appears in $d_i$ check bits, where $d_i$
is the number of edge observations involving $Y_i$. Such codes cannot be used for errorless transmission (e.g., see
our lower bound in Section \ref{s:grid}).
As a result, the techniques and results from channel coding do not appear
to apply to our setting. %Thus, our setting differs from that of channel coding. 

{\bf Correlation Clustering (CC):} There are numerous variants of this problem, but in the typical setting
$\Y$ is a partition of $N$ variables into an unknown number of
clusters and $\X_{u,v}$ specifies whether $Y_u$ and $Y_v$ are in the same cluster (with some probability of error as in \cite{JH05} or adversarially as in \cite{mathieu2010correlation}).
The goal is to find $\Y$ from $\X$. %There are several key differences
%form our setting.
Most CC
works assume an unrestricted number of %ACN08, CSW08
clusters~\cite{bansal2004correlation,JH05}, although a few
consider a fixed number of clusters~(e.g. see \cite{GG06}). %CSW08
Our results apply to the case of two clusters.
%Second, we focus on $X$ corresponding to sparse graphs
%(e.g., grids) whereas most  CC papers consider complete
%graphs~\cite{CGW05,D+06,S04}.   
% ABOVE isn't right -- CGW05, D+06,S04  consider **general** GRAPHS
 \ignore{
The worst-case polynomial-time
approximability of correlation clustering has been extensively studied
for the objective functions of maximizing the number of edge
agreements  and minimizing the number of edge disagreements.  Most
works study complete graphs and an unrestricted number of
clusters~\cite{ACN08,bansal2004correlation,CSW08}, though there are a few 
papers on a fixed number of clusters~\cite{CSW08,GG06} and general
graphs~\cite{CGW05,D+06,S04}.  
}
The most significant difference is that most of the CC works study the objective of minimizing the
number of edge disagreements. It is not obvious how to translate
the guarantees provided in these works to a non-trivial bound on
Hamming error (i.e., number of {\em node} disagreements) for our
analysis framework. 
\ignore{ Indeed,
our main technical results can be interpreted as proofs that, at least
in many interesting cases, optimally solving the %given noisy
correlation clustering problem yields classification error close to
the information-theoretic minimum.  Understanding how generally this
statement holds is an open question.  Obviously, the situation is even
murkier when the given correlation clustering problem is only solved
approximately.
}

{\bf Approximately Stable Clusterings:}
Work on approximation stability, initiated by Balcan et
al.~\cite{BBG13} and Bilu and Linial~\cite{BiluL12}, also seek
polynomial-time algorithms with low Hamming error with respect to a
ground truth clustering.  Instead 
of assuming that the input is derived from the ground truth by a
random process, these papers make an incomparable assumption that all
near-optimal solutions w.r.t.\ some objective function have low error
w.r.t.\ the ground truth clustering.  
Approximation
stable instances of correlation clustering problems were studied by
Balcan and Braverman~\cite{BB09}, who gave positive results when $G$
is the complete graph and stated the problem of understanding general
graphs as an open question.

{\bf Recovery Algorithms in Other Settings:}  The high-level
goal of recovering ground truth from a noisy input has been studied 
in numerous other application domains.  %We do not
%attempt a survey here, and note only that 
In the overwhelming majority
of these settings, the focus is on maximizing the probability of
exactly recovering the ground truth, a manifestly impossible goal in
our setting.
This is the case with, for example, 
planted cliques and graph partitions
(e.g.~\cite{condonkarp,feigekillian,mcsherry2001}),
detecting hidden communities \cite{anandkumar2013tensor}, and 
phylogenetic tree reconstruction \cite{daskalakis2006optimal}. %,erdos1999few}).
A notable exception is work by Braverman and Mossel~\cite{BM08} on
sorting from noisy information, who give polynomial-time algorithms
for the approximate 
recovery of a ground truth total ordering given noisy pairwise comparisons.  Their approach,
similar to the present work, is to compute the maximum
likelihood ordering given the data, and prove that the expected
distance between this ordering and the ground truth ordering is
small.

{\bf Recovery on Random Graphs:} Two very recent works
\cite{AbbeBBS14,ChenG14} have addressed the case where noisy pairwise
observations of $\Y$ are obtained for edges in a graph. In both of
these, the focus is mainly on guarantees for random graphs (e.g.,
Erd\"os-Renyi graphs). Furthermore, the analysis is of perfect recovery (in the
limit as $n\to\infty$) and its relation to the graph ensemble. The goal
of our analysis is considerably more challenging, as we are interested
in the Hamming error for finite $N$. Abbe et al. \cite{AbbeBBS14}
explicitly state partial (as opposed to exact) recovery for sparse
graphs with constant degrees as an open problem, which we solve in
this paper.
%Section~\ref{s:ext} of our paper solves the open problem raised

{\bf Percolation:} Some of the technical ideas in our study of
grid graphs (Section~\ref{s:grid}) are inspired by arguments in
percolation, the study of connected clusters in random (often
infinite) graphs. 
% NOTE (from David & Cafer): the below isn't quite true, so commenting out.
% For example, the fact that the critical threshold
%for bond percolation in the two-dimensional lattice is $\tfrac{1}{2}$
%suggests that, when the noise parameter $p$ is less than
%$\tfrac{1}{2}$, there should be no ``correlation decay'' and
%good approximate recovery should be possible in principle.
For example, our use of ``filled-in regions'' in Section~\ref{s:grid} is % the proof Theorem~\ref{t:gridub}
reminiscent of arguments in percolation theory (e.g., see p.~286 in \cite{grimmett99}).
In addition, we can directly adapt results from statistical physics that
bound the connectivity constant
of square lattices \cite{clisby12, madras93} and the number of
self-avoiding polygons of a particular length and area \cite{jensen00},
to give precise constants for our theoretical results.

\section{Preliminaries}\label{ss:ar}

% TODO: change ground truth to look like a foreground/background figure?
\begin{figure}[t]
\centering 
\vspace{-4mm}\includegraphics[width=.68\textwidth]{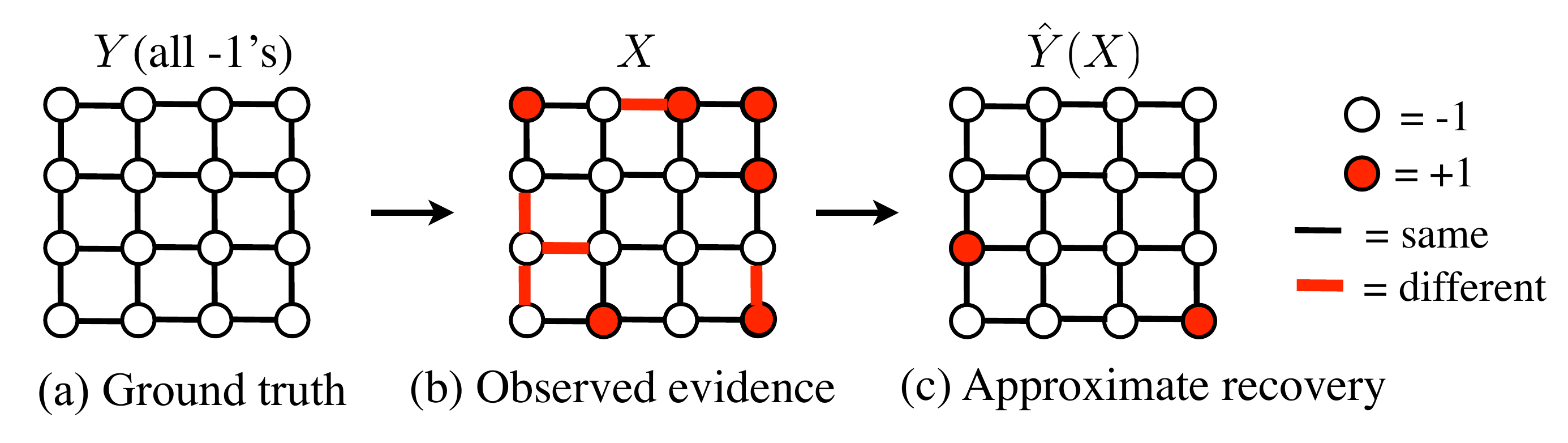}\vspace{-5mm}
\caption{\label{fig:main_figure}
\small Statistical recovery on a grid 
  graph. 
{\bf (a)} Ground truth, which we want to recover. 
 %, such as a foreground/background segmentation of 
%an image. %; in this example the ground truth is the all 0's assignment.
{\bf (b)} A possible set of noisy node and edge observations. % both
                                % at the individual nodes  
%(e.g. each pixel's affinity to foreground versus background) and along the 
%edges (e.g., whether two neighboring pixels' labels are the 
%same).
{\bf (c)}
Approximate recovery (prediction), in this case with Hamming error 2.
% {\bf (d)}
%Average error per node on a 20x20 grid (1000 trials), for statistical recovery using an approximate MAP inference algorithm (MPLP v2 
%\cite{SontagChoeLi_uai12, SontagEtAl_uai08}). Despite having a noise rate 
%of 8\% for the edge evidence and 45\% for the node evidence, MPLP 
%does a remarkably good job at approximate recovery.
}
\vspace{-3mm}
\end{figure}

We consider the setting of prediction on a graph $G=(V,E)$
where $V$ denotes the set of labels that we want to predict, and $E$
the observed pairwise relationships. 
%The graph $G=(V,E)$ is fixed and 
%known.  We can assume without loss of generality that~$G$ is connected. 
%For example, if $V$ is the set of pixels of a square image and 
%$E$ is the set of pairs of neighboring pixels, then $G$ is a grid 
%graph. 
Let $\Y \in \{-1,
+1\}^N$ denote the ground truth labels, where $N=|V|$. 
The setting is depicted in Figure~\ref{fig:main_figure}.

{\bf The Generative Model and Hamming Error:} 
A random process generates observations for the edges and nodes of $G$ as a function of 
the ground truth. It has two parameters, an edge noise $p \in
[0, .5]$ and a node noise $\q \in [0, .5]$. The generative model is as
follows. For each edge $(u,v)\in E$, the edge observation $X_{uv}$ is
independently sampled to be $\Y_u\Y_v$ with probability $1-p$ (called a {\em good} edge), and $-\Y_u\Y_v$ with probability
$p$ (called a {\em bad} edge).
Observe that adjacent
vertices are likely to have the same (or different) labels if the
observation on the edge connecting them is $+1$ (or $-1$). 
Similarly, for each node $v \in V$, the node observation $\X_v$ is
independently sampled to be $\Y_v$ with probability $1-\q$ ({\em good}
nodes), and $-\Y_v$ with probability $\q$ ({\em bad} nodes). 
%Formally, the probability of observations conditioned on ground truth assignments are
%\begin{eqnarray}
%\prob {\X_v \mid \Y_v} = (1-\q)^{\frac{1}{2} (1+\X_v\Y_v)} \q^{\frac{1}{2} (1-\X_v\Y_v)}\\
%\prob {\X_{uv} \mid \Y_u,\Y_v} = (1-p)^{\frac{1}{2} (1+\X_{uv}\Y_u\Y_v)} p^{\frac{1}{2} (1-\X_{uv}\Y_u\Y_v)}
%\end{eqnarray}

\comment{
For a noise parameter $p 
\in [0,\tfrac{1}{2}]$, each edge of $E$ is independently designated 
    {\em bad} (with probability $p$) or {\em good} (with probability 
    $1-p$).  A bad edge $(u,v) \in E$ is labeled by $\X_{uv}=-\Y_u\Y_v$ and a good edge labeled by $\X_{uv}=\Y_u\Y_v$. For another noise parameter $\q \in [0,\tfrac{1}{2}]$, each node of $v \in V$ is labeled by $\X_v=\Y_v$ with probability $1-\q$ ({\em good} nodes) and $\X_v=-\Y_v$ with probability $\q$ ({\em bad} nodes). Hence, a $+1$ label on an edge indicates that its neighboring vertices have the same label whereas a $-1$ label indicates they have the opposite label.
We emphasize that an input offers no explicit hints as to which 
edges or nodes are good and which are bad --- it is simply the graph $G$
with ``$-1/+1$'' labels on the edges and ``$-1/+1$'' labels at the nodes. We denote the random variable corresponding to the labeling of all edges and nodes by $\X$. 
}
A labeling algorithm is a function $\algo:\{-1,+1\}^{E} \times \{
-1,+1 \} ^{V} \to \{-1,+1\}^V$ from graphs with labeled edges and
nodes (i.e., the noisy observations described above) to a labeling of the nodes $V$. We
measure the performance of $\algo$ by the expectation of the Hamming
error (i.e., the number of mispredicted labels) %with respect to the
%ground truth $\Y$
over the observation distribution induced by $\Y$. By the {\em error} of an algorithm, we mean 
its worst-case (over $\Y$) expected error (over inputs generated by 
$\Y$).
\ignore{In the basic model, where there is no information to 
disambiguate the outputs $\Yhat$ and $1-\Yhat$, by necessity we measure error 
with respect to the better of these two outputs; equivalently, we 
identify the ground truths $\Y$ and $1-\Y$.\footnote{When the input also 
  comes with (noisy) node labels, identifying $\Y$ with $1-\Y$ neither makes 
  sense nor is required for approximate recovery, see 
  Section~\ref{s:ext}.}}
%
% TODO: get uppercase and lowecase notation straight.
Formally, we denote the error of the algorithm given a value $Y=y$ by $e_y(\algo)$ and define it as:
\begin{equation}\label{eq:gtruths}
e_y({\algo}) = \expect[\X \mid \Y=\y]{\tfrac{1}{2}\left\|\algo(\X)-\y\right\|_1}.
%e_y({\algo}) = \expect[\X \mid \Y=\y]{\min\left\{\left\|\algo(\X)-\y\right\|_1, \,\left\|\algo(\X)-(1-\y)\right\|_1\right\}}. 
\end{equation}
The overall error is then:
\be \label{eq:definition_of_error}
e({\algo}) = \max_\y e_\y(\algo).
\ee 

%NOTE BY TR: Note newly precise version of the main definition,
% with the order of quantifiers clarified.

{\bf MAP and Marginal Estimators:}
The maximum likelihood (ML) estimator of the ground truth is 
given by $\Yhat \leftarrow \arg\max_{\Y} p(\X \mid \Y)$, where
\be
p(\X \mid \Y) = \prod_{uv \in E} {(1-p)^{\frac{1}{2}  (1+\X_{uv}\Y_u\Y_v)} p^{\frac{1}{2}  (1-\X_{uv}\Y_u\Y_v)}} \cdot \prod_{v\in V} {(1-\q)^{\frac{1}{2} (1+\X_v\Y_v)} \q^{\frac{1}{2} (1-\X_v\Y_v)}}.
\ee
Taking the logarithm and ignoring constants, we see that maximizing $p(\X \mid \Y)$ is equivalent to
\be
\label{eq:map}
\max_{Y} \quad \sum_{uv \in E} \frac{1}{2} \X_{uv}\Y_u\Y_v \log {\frac{1-p}{p}} + \sum_{v\in V}  \frac{1}{2} \X_u\Y_u \log {\frac{1-\q}{\q}},
\ee
or simply $\max_Y \; \sum_{uv \in E} \X_{uv}\Y_u\Y_v + \gamma \sum_{v\in
  V} \X_u\Y_u$, where $\gamma= \log {\tfrac{1-\q}{\q}} / \log {\tfrac{1-p}{p}}$.
%where $\gamma= \frac{\log {\frac{1-\q}{\q}}}{\log {\frac{1-p}{p}}}$

Assuming a uniform prior over ground truths $\Y$, MAP inference
reduces to maximum likelihood inference, 
and marginal inference can be performed using $p(\Y \mid \X) \propto p(\X \mid \Y) $.
 %Hence, by the discussion above
% NOTE: below isn't quite correct, because the temperature is wrong --
% can't ignore the multiplicative constants in lg-space.
%\be
%\prob {\Y \mid \X} \propto \exp{\left( \sum_{uv \in E} \X_{uv}\Y_u\Y_v + \gamma \sum_{v\in V}  \X_u\Y_u \right)}
%\ee
%In the supplementary, we 
Standard arguments prove that the algorithm that performs
marginal inference using a  
uniform distribution over $\Y$ achieves the smallest possible error 
according to Eq.~\ref{eq:definition_of_error}; for completeness, we
include a proof in Appendix~\ref{ss:MI_minimax}.
In other words,
marginal inference using a uniform prior minimizes the worst case expected error (i.e., it is minimax optimal).

%We will use this result 
%in the Section~\ref{sec:empirical} to show empirically that both MAP 
%inference and the 2-step algorithm proposed above achieve nearly 
%optimal recovery rates. 

{\bf Approximate Recovery:}
The interesting regime for structured prediction is when the node noise $\q$ is 
large. In this regime there is no correlation decay, and correctly
predicting a label requires a more global consideration of the node
and edge observations.
The intriguing question --- and the question that reveals the 
importance of the structure of the graph $G$ --- is whether or not 
there are algorithms with small error when the edge noise $p$ is a small 
constant. 
Precisely, we make the following definition.
\begin{definition}[Approximate Recovery]
For a family of graphs $\G$, we say that {\em approximate 
  recovery is 
  possible} if there is a function $f:[0,1] \rightarrow [0,1]$ with 
$\lim_{p \downarrow 0} f(p) = 0$ such that, for every sufficiently small 
$p$ and all $N$ at least a sufficiently large constant $N_0(p)$,
the minimum-possible error of an algorithm on a graph $G \in \G$ with 
$N$ vertices is at most 
$f(p) \cdot N$. 
\end{definition}

%All of the following are well-motivated goals:
%\begin{enumerate}

%\item 
%Characterize the families of graphs for which 
%approximate recovery is possible. 

%\item 
%For graph families that allow approximate recovery, especially those 
%common in applications, pin down how small 
%the function $f$ can be. 

%\item 
%Devise polynomial-time algorithms for 
%approximate recovery in such graphs. 

%\end{enumerate}

\comment{
\subsection{Path Graphs: Approximate Recovery Is Not Always Possible \label{ss:path}}}

{\bf A Non-Example:}
Some graph families admit approximate recovery whereas others do not. 
To illustrate this and impart some intuition about
our model, consider the family of path graphs.
Assume that the node noise $\q$ is extremely close to $.5$, so that node
labels provide no information about the ground truth, while the
edge noise $p$ is an arbitrarily small positive constant.
If $G$ is a path graph on $N$ nodes with $N$ sufficiently large then,
with high probability, for most pairs of nodes, the unique path 
  between them contains a bad edge.  This 
implies that approximate recovery is not possible. 

A bit more formally, imagine that an adversary generates the ground 
truth $\Y$ by picking $i$ uniformly at random from $\{1,2\ldots,N\}$,
giving the first $i$ nodes the label -1 and the last $N-i$ nodes the 
label +1.  With high probability a constant fraction of the 
input's edges are ``-1'' edges --- one good edge consistent 
with the ground truth and the rest bad edges inconsistent with the 
ground truth.  Intuitively, no algorithm can guess which is which,
which means that every algorithm has expected error $\Omega(N)$ with 
respect to the distribution over $\Y$, and hence error $\Omega(N)$ with 
respect to a worst case choice of $\Y$.  Thus, path graphs do not 
allow approximate recovery.\footnote{It is not difficult to make this 
  argument rigorous.  See Section \ref{sec:lb} for a rigorous, and more 
  interesting, version of this lower bound argument.}

% ----------------------------

\comment{
\subsection{Marginal Inference is Minimax Algorithm}
}

\section{Optimal Recovery in Grid Graphs \label{s:grid}}

%Section~\ref{ss:mot} explains why grid graphs deserve special
%treatment.  Section~\ref{s:ub} gives a polynomial-time for approximate
%recovery in grid graphs with expected error $O(p^2N)$.  The analysis
%is split over Sections~\ref{s:ub} and~\ref{ss:eval}.
%Section~\ref{sec:lb} gives an information-theoretic lower bound of
%$\Omega(p^2N)$.

%\subsection{Motivation}\label{ss:mot}

%\paragraph{Motivation.}
This section studies grid graphs.  We devote a
lengthy treatment to them for several reasons.  First, grid graphs are
central in applications such as machine vision.
%Second, our
%analysis of grid graphs extends to all planar graphs that share two
%features with the grid, a weak expansion property and a bounded-degree
%dual graph (see Section~\ref{ss:gridlike}).
Second, the grid is a
relatively poor expander and for this reason poses a
number of interesting technical challenges.
% that are not present in
%graphs that are reasonably good expanders.  
%The analysis in
%Section~\ref{ss:expanders} fails to carry over to grids because a set $S$
%of misclassified nodes with boundary size $|\delta(S)|$ might have
%size as large as $|S|^2$.  In a graph with a large minimum cut, this
%challenge can be addressed by proving good upper bounds on the number
%of sets that can possess a given boundary size
%(Section~\ref{ss:bigcut}).  Unfortunately, these bounds fails to hold
%in grids.  Instead, we use planarity properties to argue that it is
%sufficient to control the error contributions of a relatively small
%and carefully chosen collection of sets.  
%The fourth reason to care
%about grid and other planar graphs is that 
%
Third, our algorithm for the grid and other planar graphs
is computationally efficient.
%
%approximate recovery is
%possible not only in principle, but also computationally efficiently.
Our grid analysis yields matching
upper and lower bounds of $\Theta(p^2N)$ on the
information-theoretically optimal error.

%%\subsection{Overview}

%In what follows we analyze upper and lower bounds on labeling error in
%the grid graph case. The upper bound in Section \ref{s:ub} is obtained
%using a simple polynomial time algorithm. The lower bound is presented
%in Section \ref{sec:lb}. 

\comment{
\subsection{Preliminaries}
Let $\Yhat$ denote the partition of $V$ into two sets that maximizes the
number of ``-'' edges that get cut plus the number of ``+'' edges that
don't get cut.  
Label all nodes of one side 0 and all nodes of the other side 1; among
the two possibilities, choose the
one that leads to Hamming error at most $N/2$ w.r.t.\ $\Y$.
Let $H$ denote the Hamming error between $\Y$ and $\Yhat$.
}

%\paragraph{Algorithm.} 

\subsection{The Algorithm}

We study the algorithm $\ouralg$, which has two stages.
The first stage ignores the node observations and computes a
labeling $\Yhat$ that maximizes the agreement with respect to edge
observations only, i.e.
\begin{equation}\label{eq:obj_first}
	\Yhat \leftarrow \arg\max_{\Y} \sum_{uv \in E} \X_{uv}\Y_u\Y_v.
\end{equation}
Note that $\Yhat$ and $-\Yhat$ agree with precisely the same set of edge
observations, and thus both maximize Eq.~\ref{eq:obj_first}.  The second stage of
algorithm $\ouralg$ outputs $\Yhat$ or $-\Yhat$, according to a
``majority vote'' by the node observations. Precisely, it outputs
$-\Yhat$ if $\sum_{v\in V} \X_v\Y_v < 0$, and $\Yhat$ otherwise.

\comment{
We call a set $S$ {\em maximal} if (1) $S$ is connected; (2) all nodes in $S$ predicted correctly or incorrectly by $\Yhat$ with respect to ground truth $\Y$; (3) there is no other subset $S'$ with $S \sse S'$ satisfying first two property. Note that Flipping Lemma works for all maximal sets $S$. The intuition of our proof is that there we can control the error contribution by type1-5 sets by using fill-in argument.
}
%\comment{
\begin{algorithm}
\caption{$\ouralg(\X)$\label{alg:2-step}}
\begin{algorithmic}
\REQUIRE Edge and node observations $\X$
\ENSURE Node predictions $\Yhat$
\STATE $\Yhat \leftarrow \arg\max_{\Y} \sum_{uv \in E} \X_{uv}\Y_u\Y_v$
\IF{$\sum_{v\in V} \X_v\Y_v < 0$}
	\STATE $\Yhat \leftarrow -\Yhat$
\ENDIF	
\RETURN $\Yhat$
\end{algorithmic}
\end{algorithm}
%}
%
When the graph $G$ is a 2D grid, or more
generally a planar graph, this algorithm can be implemented in
polynomial time by a reduction 
to the maximum-weight matching problem (see
\cite{fisher1966dimer,barahona}).  
By contrast, it is $NP$-hard to maximize the full expression
in~\eqref{eq:map}~\cite{barahona}.

\subsection{An Upper Bound on the Error \label{s:ub}}
Our goal is to prove the following theorem, which shows that
approximate recovery on grids is possible.
\begin{theorem}
\label{thm:ub}
If $p<1/39$, then the algorithm $\ouralg$ achieves error
$e(\ouralg)=O(p^2 N)$. 
\end{theorem} 
%We prove a matching lower bound at the end of the section.

{\bf Analysis of First Stage:}
We analyze the two stages of algorithm $\ouralg$ in order.  We first
show that after the first stage, the expected error of the better
of  $\Yhat,-\Yhat$ is $O(p^2 N)$.  We then extend this error bound
to the output of the second stage of the algorithm.

%small, and then show t
%We begin by presenting an algorithm $\ouralg$ for the labeling problem
%on grid graphs. We will then show that $e(\ouralg)=O(p^2 N)$. The
%algorithm is same as the one we analyzed for expanders
%(Section~\ref{ss:expanders}):  
%given a labeling $\X$ of the edges,
%find a labeling $\Yhat$ of the nodes that maximizes
%\be
%\sum_{e\in E: e=(u,v)} Z_e \left(Y_u-0.5\right)\left( Y_v-0.5\right)
%\ee
%In other words, we are seeking a labeling that is in maximal 
%agreement with the edge labels $\X$ (where an agreement is giving the
%endpoints of a ``+'' edge the same label of the endpoints of a
%``-'' edge different labels).

We begin by highlighting a simple but key lemma that characterizes a
structural property of the maximizing assignment in Eq.~\ref{eq:obj_first}.
% The point of the lemma is to ``charge'' node labelling
%errors by the algorithm to edges that were corrupted.
We use $\delta(S)$ to denote the boundary of $S \sse V$, i.e. the
set of edges with exactly one endpoint in $S$.
\begin{lemma}[Flipping Lemma]\label{l:flipping}
Let $S$ denote a maximal connected subgraph of $G$ with every node of
$S$ incorrectly labelled by $\Yhat$ or $-\Yhat$.  Then at least half the edges
of $\delta(S)$ are bad.
\end{lemma}
\begin{proof}
The computed labeling $\Yhat$ (or $-\Yhat$) agrees with the edge
  observations on at least half the edges of $\delta(S)$ ---
  otherwise, flipping the labels of all nodes in $S$ would yield a new
  labeling with agreement strictly higher than $\Yhat$ (or $-\Yhat$).
On the other hand, since $S$ is maximal, for every edge $e \in \delta(S)$,
exactly one endpoint of $e$ is correctly labeled.  Thus {\em every} edge
of $\delta(S)$ is inconsistent with the ground truth.  These two
statements are compatible only if at least half the edges of
$\delta(S)$ are bad.
\end{proof}

%Let $S$ denote the nodes of $V$ correctly
%classified by $\Yhat$ and $C_1,\ldots,C_k$ the connected components of
%the (misclassified)
%nodes of the induced subgraph $G[V \sm S]$.  The claim is that at
%least half the edges in each boundary $\delta(C_i)$ are bad.  Why?
%Since all nodes of $C_i$ are mislabelled and all of the neighbors of
%$C_i$ are correctly labelled, all edges of $\delta(S)$ are inconsistent
%with respect to the ground truth.  On the other hand, at least half
%the edges of $\delta(S)$ are consistent with the input labels ---
%otherwise, flipping the labels of all nodes in $C_i$ would yield a
%better correlation clustering of the input, which contradicts the optimality o%f $\Yhat$. These two facts imply
%that at least half of the edges in $\delta(C_i)$ are bad.

Call a set $S$ {\em bad} if at least
half its boundary $\delta(S)$ is bad.
The Flipping Lemma motivates bounding the probability that a given set
is bad, and then enumerating over sets $S$.  This approach can be made
to work only if the collection of sets $S$ is chosen carefully ---
otherwise, there are far too many sets and this approach fails to
yield a non-trivial error bound.

To begin the analysis,
%The rest of this section and the next are devoted to the proof.  We use 
let $H$ denote the error of our algorithm on a random input.
$H$ seems difficult to analyze directly, so we introduce a
simpler-to-analyze upper bound.
This requires some definitions.
Let $\C$
denote the subsets $S$ of $V$ such that the induced subgraph $G[S]$ is
connected.  We classify subsets $S$ of $\C$ into 6 categories (see Figure~\ref{fig:type1-6}):
\begin{enumerate}
\item $S$ contains no vertices on the perimeter of $G$;
\item $S$ contains vertices from exactly one side of the perimeter of
$G$;
\item $S$ contains vertices from exactly two sides of the perimeter of
$G$, and these two sides are adjacent;
\item $S$ contains vertices from exactly two sides of the perimeter of
$G$, and these two sides are opposite;

\item $S$ contains vertices from exactly three sides of the perimeter of
$G$;
\item $S$ contains vertices from all four sides of the perimeter         
  of $G$.
\end{enumerate}
\begin{figure}[!h]
\centering
\includegraphics[width=.6\textwidth]{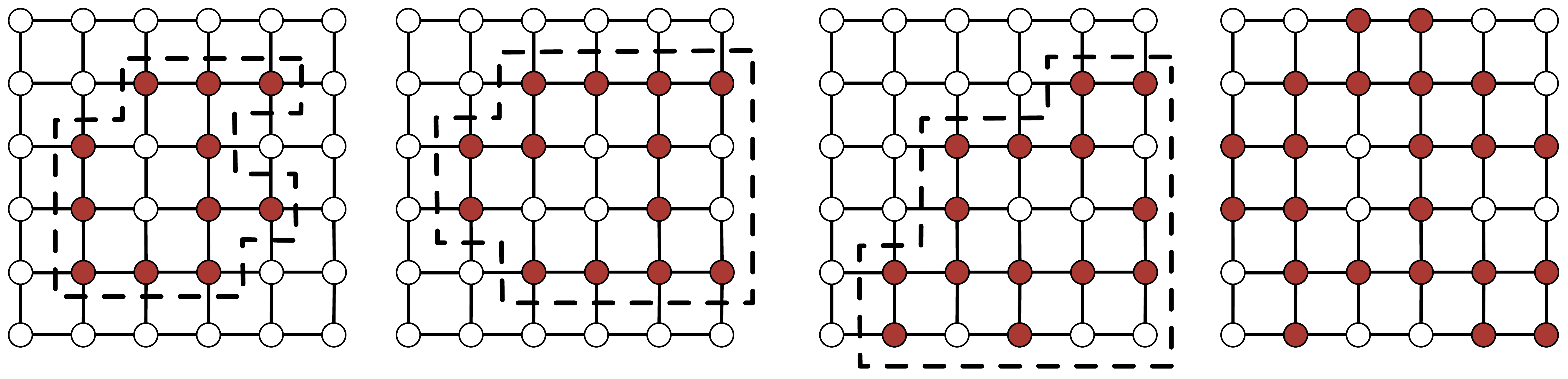}
\vspace{-3mm}
\caption{Examples of type 1, 2, 3, and 6 regions, left-to-right. \label{fig:type1-6}}
\end{figure}
%\begin{itemize}
%
%\item [(1)] $S$ contains no vertices on the perimeter of $G$.
%
%\item [(2)] $S$ contains vertices from exactly one side of the perimeter of
%  $G$.
%
%\item [(3)] $S$ contains vertices from exactly two sides of the perimeter of
%  $G$, and these two sides are adjacent.
%
%\item [(4)] $S$ contains vertices from exactly two sides of the perimeter of
%  $G$, and these two sides are opposite.
%
%\item [(5)] $S$ contains vertices from exactly three sides of the perimeter of
%  $G$.
%
%\item [(6)] $S$ contains vertices from all four sides of the perimeter
%  of $G$.
%
%\end{itemize}
Let $\C_{<6}$ denote the set of all $S\subset V$ from one of the first 5 categories.  For
a set $S\in \C_{<6}$, we define a corresponding {\em filled in} set
$F(S)$.  Consider the connected components $C_1,\ldots,C_k$ of
$G[V \sm S]$ for such a subset $S$.  
Call such a connected component {\em 3-sided} if it includes vertices
from at least three sides of the grid $G$.
For every $S \in \C_{<6}$ there is at least one 3-sided component; it
is unique if $S$ has type 1, 2, 3, or 5.
We define $F(S)$ as the union of $S$ with all the connected components of
$G[V \sm S]$ except for a single 3-sided one.  
%(We use an arbitrary
%rule to break ties among multiple 3-sided components.) 
Appendix~\ref{sec:illustration_filled_in_sets}
illustrates the filling-in procedure.
$F(S)$ is not defined for type-6
components $S$. 
Observe that $F(S) \supseteq S$.
%Observe that both $F(S)$ and its complement induce connected subgraphs.
Let $\F = \{ F(S) : S\in \C_{<6}\}$ denote the set of all such
filled-in components.

%We next work toward a proof of 
%Lemma~\ref{l:ub} follows easily from the next two lemmas.
\begin{lemma}\label{l:disjoint}
If $S_1,S_2$ are disjoint and not type 6, then $F(S_1),F(S_2)$ are
distinct and not type 6.
\end{lemma}
\begin{proof}
If a set $S$ is not type 6, then every 3-sided component of $G[V \sm
  S]$ contains one entire side of the grid perimeter.  Since $F(S)$
  excludes a 3-sided component, it cannot be type 6.

Also, for a set $S$ that is not type 6, the boundary of $F(S)$ is a
non-empty subset of that of $S$.  Thus, the non-empty set of endpoints
of $\delta(F(S))$ that lie in $F(S)$ also lie in $S$.  This implies
that if $F(S_1) = F(S_2)$, then $S_1 \cap S_2 \neq \emptyset$.
\end{proof}
\comment{
Now let $B$ denote the mislabeled vertices of the better of $\Yhat$, $-\Yhat$ and let
$B_1,\ldots,B_k$ denote the connected components of $G[B]$. We can
assume that none of the $B_i$'s are type 6 by  
outputting $-\Yhat$ if necessary. This can be done since our error
definition is the better of the $\Yhat$ and $-\Yhat$ (see~\eqref{eq:gtruths}).  
If $\Yhat$ has a mislabeled type 6 component then it would be correctly
labeled if we output $1-\Yhat$.  
Moreover, if there is a correctly labeled type 6 component there cannot be any mislabeled type 6 
components since correctly labeled type 6 component contains a path from left side to right side 
preventing any mislabeled component to connect bottom side to top side.
}
%We define another predictor $\Yhat _0$ to be $\Yhat$ if $\Yhat$ does
%not incorrectly classify a type-6 set, as $-\Yhat$ if $\Yhat$
%incorrectly classify a type-6 set. Note that $\Yhat _0$ is also
%optimal. Hence, Flipping Lemma works for $\Yhat _0$ as well. It is
%easy to see that the Hamming error of $\Yhat _0$ is larger than the
%error of the better of $\Yhat,-\Yhat$. Hence, it suffices to upper
%bound the Hamming error of $\Yhat _0$ by $O(p^2 N)$.
The following error upper bound applies to whichever of $\Yhat,-\Yhat$
does not incorrectly classify a type-6 set 
(there is at most one type-6 set, so at least one of them has
this property).
Let $B$ denote the mislabeled vertices of %$\Yhat _0$ 
such a labeling and let
$B_1,\ldots,B_k$ denote the connected components (of types 1--5) of $G[B]$. 
%By definition of $\Yhat _0$ none of $B_i$'s are type-6.
The next lemma extends the Flipping Lemma.
\begin{lemma}\label{l:relabel}
For every set $B_i$, the filled-in set $F(B_i)$ is bad.
\end{lemma}

\begin{proof}
We first claim that $\Yhat$ agrees with the data on at least half the
edges of $\delta(F(B_i))$; the same is true of $-\Yhat$.
The reason is that
flipping the label of every vertex of $F(B_i)$ increases the
agreement with the data by the number of disagreeing edges of
$\delta(F(B_i))$ minus the number of agreeing edges of
$\delta(F(B_i))$, and this difference is non-positive by the optimality
of $\Yhat$.

On the other hand, since $B_i$ is maximal, every neighbor of $B_i$ is
correctly labeled in $\Yhat$.  Since the neighborhood of $F(B_i)$ is a
subset of $B_i$, this also holds for $F(B_i)$.  Thus, $\Yhat$ disagrees
with $\Y$ on every edge of $\delta(F(B_i))$.
%We conclude that at least half the edges of $\delta(F(B_i))$ are bad. It is easy to see that the proof works for $-\Yhat$, since it also maximizes Eq.~\ref{eq:obj_first}
\end{proof}

A crucial point is that
Lemmas~\ref{l:disjoint} and~\ref{l:relabel} imply that
%Define a new random variable $T$ by
the random variable
\begin{equation}\label{eq:T}
T = \sum_{F \in \F} |F| \cdot 1_{F \mbox{~is bad}}
\end{equation}
is an upper bound on the error $H$ with probability 1.
We now upper bound the easier-to-analyze quantity $T$.
%the random variable~$T$ (see~\eqref{eq:T}).
%The proofs for the next three Lemmas can be found
%in the supplementary.
%The first is straightforward to prove, and it 
The first lemma provides
an upper bound on the
probability that a set~$S$ is bad, as a function of its boundary
size~$|\delta(S)|$. 
\begin{lemma}\label{l:chernoff}
For every set $S$ with $|\delta(S)| = i$,
$\prob{\mbox{$S$ is bad}} \le (3\sqrt{p})^i$.
\end{lemma}

\begin{proof}
By the definition of a bad set, $\prob{\mbox{$S$ is bad}}$
equals the probability that at least half of $\delta(S)$ are bad
edges. Since $|\delta(S)|=i$ this is the probability 
that at least $i\over 2$ edges are bad. Since these events are IID, we can bound it via:
% Chernoff inequality. We use the form of the inequality where for $\mu,\alpha$ it holds that:\footnote{This version appears in John Canny's online notes at \url{http://www.cs.berkeley.edu/~jfc/cs174/lecs/lec10/lec10.pdf}.}
%\be
%\prob\left[\sum_j X_j > (1+\alpha)\mu  \right] < \left({e^{\alpha}\over (1+\alpha)^{(1+\alpha)}}\right)^{\mu}
%\ee
%Taking $\mu= i p $, and $\alpha={{1\over 2p} -1}$. Then this becomes:
\be
%P\left[\sum_j X_j > {i\over 2}  \right] < \left({e^{{{1\over 2p} -1}}\over ({1\over 2p})^{1\over 2p}}\right)^{ip} = \frac{e^{0.5i-ip}}{({1\over 2p})^{i \over 2}} = e^{0.5i-ip} (2p)^{i\over 2} <  \left(\sqrt{2 e p}\right)^i < \left(3\sqrt{p}\right)^i
\prob{\sum_j Z_j \geq {i \over 2}} <  {{i} \choose {i\over 2}} p^{i\over 2} \leq (2e)^{i\over 2} p^{i\over 2} \leq \left(3\sqrt{p}\right)^i
%e^{0.5i-ip} (2p)^{i\over 2} <  \left(\sqrt{2 e p}\right)^i < \left(3\sqrt{p}\right)^i
\ee
%\be
%\prob{\mbox{$S$ is bad}}=\prob{\sum_{e\in \delta(S)} \id{Z_e } 
%\ee
%See other writeup.
where $Z_j$ is the indicator event of the $i$-th edge being bad.
\end{proof}

The probability bound in Lemma~\ref{l:chernoff}
is naturally parameterized by the number of
boundary edges.  Because of this, we face two tasks in upper
bounding~$T$.  First, $T$ counts the number of {\em nodes} of bad
filled-in sets $F \in \F$, not boundary sizes.  The next lemma states
that the number of nodes of such a set cannot be more than the
square of its boundary size.
\begin{lemma}\label{l:square}
For $F \in \F$: (1) $|F| \le |\delta(F)|^2$; 
(2) if $F$ is a type-1 region, then $|F| \le                           
\tfrac{1}{16}|\delta(F)|^2$.
%\begin{itemize}

%\item [(1)] $|F| \le |\delta(F)|^2$;

%\item [(2)] if $F$ is a type-1 region, then $|F| \le
%  \tfrac{1}{16}|\delta(F)|^2$.

%\end{itemize}
\end{lemma}

\begin{proof}
If $F$ is a type 4 or 5 set, then $|\delta(F)|\geq  \sqrt{N}$ and the bound is
trivial.  If $F$ is a type
 1 set, let $U$ be the smallest rectangle in the dual graph that
 contains $F$. Let $k,m$ denote the side lengths of $U$. Then: 
$|F|\leq km \leq \frac{1}{16} (2k+2m)^2\leq
 \frac{1}{16}|\delta(F)|^2$. Similarly for type 2 sets we have
 $|F|\leq km \leq  \min \left\{ (2k+m)^2,  (k+2m)^2\right\} \leq
 |\delta(F)|^2$. Finally, for type 3 sets have $|F|\leq km \leq (k+m)^2
 \leq |\delta(F)|^2$. 
%with side lengths $k$, $m$ in dual box which contains $R$.
% 2, or 3, see the other writeup for a proof that squares are the worst case.  This immediately leads to the claimed bounds 
%(e.g., for a type 1 region, the worst case is an $l \times l$ square, which has boundary size $4l$ and area $l^2$).
\end{proof}

The second task in upper bounding $T$ is to count the number of
filled-in sets $F \in \F$ that have a given boundary size.
We do this by counting simple cycles in the dual graph.

\begin{lemma}\label{l:count}
Let $i$ be a positive integer.
\begin{itemize}

\item [(a)] 
If $i$ is odd or 2, then there are no type-1 sets $F \in \F$
  with $|\delta(F)| = i$;

\item [(b)] 
If $i$ is even and at least 4, then there are at most 
$\frac{N \cdot 4 \cdot 3^{i-2}}{2i}
= N \cdot \frac{2 \cdot 3^{i-2}}{i}
$ 
type 1 sets $F \in \F$ with $|\delta(F)| = i$;

\item [(c)] 
If $i$ is at least 2, then there are at most 
$2\sqrt{N} \cdot 3^{i-2}$
type 2--5 sets $F \in \F$ with $|\delta(F)| = i$.

\end{itemize}
\end{lemma}

\begin{proof}
Recall that, by construction, a filled-in set $F \in \F$ is such that
both $G[F]$ and $G[V \sm F]$ are connected.  This is equivalent to the
property that $\delta(F)$ is a minimal cut of $G$ --- there is
no subset $S$ such that $\delta(S)$ is a strict subset of $\delta(F)$.
In a planar graph such as $G$, this is equivalent to the property that
the dual of $\delta(F)$ is a simple cycle in the dual graph $G^d$ of $G$
(e.g., see Section 4.6 of 
%Diestel's {\em Graph Theory} 
\cite{Diestel-GraphBook})
Note that the dual graph $G^d$ is just an $(n-1) \times (n-1)$ grid
--- with one vertex per ``grid cell'' of $G$ --- plus an extra vertex $z$
of degree $4(\sqrt{N}-1)$ that corresponds to the outer face of $G$.
The type-1 sets of $\F$ are in dual correspondence with the simple cycles
of $G^d$ that do not include $z$, the other sets of $\F$ are in
dual correspondence with the simple cycles of $G^d$ that do include $z$.
The cardinality of the boundary $|\delta(F)|$ equals the length of the
corresponding dual cycle.

Part~(a) follows from the fact that $G^d \sm \{z\}$ is a bipartite
graph, with only even cycles, and with no 2-cycles.

For part~(b), we count simple cycles of $G^d$ of length $i$ that
do not include $z$.  There are at most $N$ choices for a starting
point.  There are at most 4 choices for the first edge, at most 3
choices for the next 
$(i-2)$ edges, and at most one choice at the final step to return to
the starting point.  Each simple cycle of $G^d \sm \{z\}$ is counted
$2i$ times in this way, once for each choice of the starting point and
the orientation.  

For part~(c), we count simple cycles of $G^d$ of length $i$ that
include $z$.  We start the cycle at $z$, and there are at most
$4\sqrt{N}$ choices for the first node.  There are at most 3 choices
for the next $i-2$ edges, and at most one choice for the final edge.
This counts each cycle twice, once in each orientation.
\end{proof}

Let $\F_1 \sse \F$ denote the type-1 sets of $\F$.
The computation below
shows that
\begin{equation}\label{eq:T_error}
	\expect{T} \le  cp^2N + O(p\sqrt{N})
\end{equation}
for a constant $c > 0$ that is independent of $p$ and $N$, which
completes the analysis of the first stage of the algorithm~$\ouralg$.
The intuition for why this computation works out is that
Lemma~\ref{l:count} implies that there is only an exponential
number %(in the boundary size $i$, and of the form $c^i$, not $N^i$) 
of relevant regions to sum over;
Lemma~\ref{l:square} implies that the Hamming error is quadratically
related to the (bad) boundary size; and Lemma~\ref{l:chernoff} implies
that the probability of a bad boundary is decreasing exponentially in
$i$ (with base $3\sqrt{p}$).  Provided $p$ is at most a sufficiently small
constant (independent of $N$), the probability term dominates
and so the expected error is small.

Formally, we have
\begin{eqnarray}
\label{eq:main1}
\expect{T} & = & \sum_{F \in \F} |F| \cdot \prob{\mbox{$F$ is bad}}\\
\nonumber
& = &
\sum_{i=2}^{\infty}
\sum_{F \in \F_1 \,:\, |\delta(F)| = 2i} |F| \cdot \prob{\mbox{$F$ is bad}}
+
\sum_{j=2}^{\infty}
\sum_{F \in \F \sm \F_1 \,:\, |\delta(F)| = j} |F| \cdot
\prob{\mbox{$F$ is bad}}\\
\label{eq:main2}
& \le & 
\sum_{i=2}^{\infty}
\sum_{F \in \F_1 \,:\, |\delta(F)| = 2i} \frac{i^2}{4} \cdot (3\sqrt{p})^{2i}
+
\sum_{j=2}^{\infty}
\sum_{F \in \F \sm \F_1 \,:\, |\delta(F)| = j} j^2 \cdot
(3\sqrt{p})^{j}\\
\label{eq:main3}
& \le & 
\sum_{i=2}^{\infty}
N \cdot \frac{2 \cdot 3^{2i-2}}{i}
\frac{i^2}{4} \cdot (3\sqrt{p})^{2i}
+
\sum_{j=2}^{\infty}
2\sqrt{N} \cdot 3^{j-2} \cdot 
j^2 \cdot
(3\sqrt{p})^{j}\\
\nonumber
& = & 
N \sum_{i=2}^{\infty} \frac{i}{16} (81p)^i
+
\sqrt{N} \sum_{j=2}^{\infty} \frac{2j^2}{9} (9\sqrt{p})^j\\
& = & N(cp^2) + O(p\sqrt{N}),\label{eq:T_final}
\end{eqnarray}
for a constant $c > 0$ that is independent of $p$ and $N$.  In the
derivation, \eqref{eq:main1} follows from the definition of $T$
and linearity of expectation, \eqref{eq:main2} follows from
Lemmas~\ref{l:chernoff} and~\ref{l:square}, and~\eqref{eq:main3}
follows from Lemma~\ref{l:count}.  In the final line, we are assuming
that $p < 1/81$.

\begin{remark}\label{rem:prec}
%
%\paragraph{More Details on the Constant Factor $c$}
There are several ways to optimization the computation above.
The requirement that $p < 1/81$ was needed for the infinite series to
converge. To improve this, we can use the tighter upper bound of
$(2ep)^{i/2}$ for the probability that a region of boundary size $i$
is bad (see Lemma \ref{l:chernoff}). We can then replace the upper bound
on the number of regions of each type in Lemma~\ref{l:count} with
tighter results from statistical physics. In particular, the number of
type-1 sets with boundary size $i$ can be upper bounded by $N\mu^i$
(Eq. 3.2.5 of \cite{madras93}), where $\mu$ is the so-called
connective constant of square lattices and is upper bounded by 2.65
\cite{clisby12}. The number of type 2--5 sets with boundary length $i$
can similarly be upper bounded by $4\sqrt{N}\mu^ie^{\kappa\sqrt{i}}$
for the same value of $\mu$ and for some fixed constant $\kappa>0$
\cite{hammersley62}. Putting these together, we obtain that the
infinite series for all region types is at most a constant when $p <
1/39$. 
% TODO: check constants/exponents in the upper bound of type 2--5 regions.

To compute an upper bound on the constant~$c$ in 
the term in $\eqref{eq:T_final}$ that is linear in $N$, recall
that this term can be
attributed to the type-1 regions. We expand the sum in
\eqref{eq:main1} over type-1 regions into two terms: one term that
explicitly enumerates over type-1 regions whose corresponding simple
cycle in $G^d$ is of length $i=2$ to $100$, and a remainder term. The
sum in the first term can be computed exactly as follows. For each
value of $i$, the probability that the region is bad is simply
$\sum_{k=i/2}^i {i \choose k} p^k(1-p)^{i-k}$. We can then use the
bound $\sum_{F\in \F_1 : |\delta(F)|=i} |F| \leq N
\sum_{a=1}^{i^2/16}a c_{a,i}$, where $c_{a,i}$ is the number of
distinct cycles in an infinite grid of length $i$ and area $a$ (up to
translation). These cycles also go by the name of {\em self-avoiding
  polygons} in statistical physics, and the numbers $c_{a,i}$ have
been exhaustively computed up to $i=100$ \cite{jensen00}. Finally, the
infinite sum in the remainder can be shown to be upper bounded by
$51^2 b^{51}/(1-b)^3$ for $b=2ep(2.65)^2$. The resulting function can
then be shown to be upper bounded by $8Np^2$ for $p \leq 0.017$. 
\end{remark}

% Not including: $T$ not much more than $H$ w.h.p.  
% NOTE: Not with high probability since must use Markov inequality. Prob. of failure = Pr(T>= N/2) <= 2E[T]/N, so we only get the E[H]\leq E[T] result with probability 1-2*E[T]/N. Expected Hamming error actually a bit worse in this setting if we use a weak upper bound of N on Hamming error in that setting.

%In the supplementary, we show that the constant $c$ in
%\eqref{eq:T_error} is $8$ for $p \leq 0.017$. In deriving the precise
%constant, we make use of results from statistical physics on the
%connectivity constant of square lattices \cite{clisby12, madras93}, and of exp%licit computations
%of the number of self-avoiding polygons (which correspond to our filled in reg%ions) of
%a particular boundary length and area \cite{jensen00}.

%\subsection{Noisy Node Labels}\label{ss:node}
{\bf Analyzing the Second Stage:}
%A natural extension of the basic model is to include noisy node labels.  
%If $\Y$ denotes the ground truth, then each node $v \in V$ of the input 
%is labelled $\Y_v$ with probability $1-q$ and $1-\Y_v$ with probability 
%$q < \tfrac{1}{2}$.  Edges of $G$ are labelled as usual, according to 
%the noise parameter $p$. 
%It no longer makes sense to identify the ground 
%truths $\Y$ and $-\Y$, and we no longer take the better of the two 
%classifications $\Yhat,1-\Yhat$ induced by the cut computed by our algorithm. 
%
%Approximate recovery with noisy node labels reduces to approximate 
%recovery without them, provided $q$ is at least slightly below 
%$\tfrac{1}{2}$.  To see why, consider solving the more general 
%problem in two phases. 
%In the first phase, we ignore the node labels 
%and run the algorithm which maximizes agreements with edge labels and
%obtain the classification $\Yhat$, as in Sec~\ref{ss:eval}. By the
%analysis we have done in Section~\ref{ss:eval}, one of  
Our analysis so far shows that
the better of $\Yhat,-\Yhat$ has small error with respect to the
ground truth $\Y$.  In  
the second phase, we use the node labels to 
choose between them via a ``majority vote.'' 
%as to which one of $\Yhat,1-\Yhat$ is a better fit to the data. 
We next show that,
%Straightforward Chernoff bounds imply that, 
provided $q$ is slightly 
below $\tfrac{1}{2}$, the better of $\Yhat,-\Yhat$ is chosen in the 
second stage with high probability.
This 
%implies that the approximate recovery bound of the 
%original algorithm without node labels carries over to the two-phase 
%algorithm with node labels, which proves 
completes the proof of Theorem~\ref{thm:ub}.
%The formal proof of the second stage appears in the supplementary.

Our starting point for the second-stage analysis is
the inequality $\expect {H_0} \leq N \cdot cp^2$, where $H_0$ is the Hamming error of the better of 
$\Yhat,-\Yhat$.
%However, since we have node observations now, we are
%not guaranteed to choose better of $\Yhat,-\Yhat$. We will use node
%observations to choose better flip. First, by 
Markov's inequality implies that $\prob {H_0\geq  
  \frac{1}{k p^2} Ncp^2 } \leq 
k p^2$, where $k$ is a free parameter.

For the second stage,
let $B'$ be the set of wrong node observations. Chernoff bounds imply
that, for every constant $\delta > 0$ and sufficiently large $N$,
$\prob{|B'|\geq (1+\delta)Nq} \leq \frac{1}{N^2} $. 
Observe that if the sum of the number of bad node observations and
the number of misclassified nodes for the better of $\Yhat,-\Yhat$ is
less than $N/2$, then the two-stage algorithm~$\ouralg$ would choose
the better of 
$\Yhat,-\Yhat$. Hence, with probability $1- k p^2 -\frac{1}{N^2}$, the
algorithm would choose the better of $\Yhat,-\Yhat$  provided $ \frac{1}{k p^2}
Ncp^2 + (1+\delta) Nq <\frac{N}{2}$, or equivalently, 
$$ \frac{c}{k} +
(1+\delta) q <\frac{1}{2}.$$ 
This inequality is satisfied for small $\delta$ provided
$k>\frac{c}{1/2 - (1+\delta)q}$.
Thus,  
$$ \expect{H} \leq 1 \cdot N cp^2 + (k p^2 + \frac{1}{N^2}) \cdot N \leq N \cdot ( (c+1) p^2 + k p^2 ) \leq N \cdot C p^2 $$
for $ N > N_0(p,q)$, where $H$ is the error of the 2-step algorithm. 
(In second inequality we use that $N > \frac{1}{p}$.)

\subsection{Lower Bound \label{sec:lb}}

%{\bf Lower Bound (Sketch):}
% TODO: give Theorem statement
%
In this section, we prove that every algorithm
suffers worst-case (over the ground truth) expected error
$\Omega(p^2N)$ on 2D grid graphs, matching the upper bound for the
2-step algorithm~$\ouralg$ that we proved in Theorem~\ref{thm:ub}. We use the
fact that marginal inference is 
minimax optimal for Eq.~\ref{eq:definition_of_error} 
(see Appendix~\ref{ss:MI_minimax}).
The expected error of marginal inference is
independent of the ground truth (by symmetry), so we can lower bound
its expected error for the all-0 ground truth.
Also, its error only decreases if it is given part of the ground truth.

Let $G=(V,E)$ denote an $n \times n$ grid with $N = n^2$ vertices.
Let $\Y:V \rightarrow \{-1,+1\}$ denote the ground truth.  
We consider the case where $\Y$ is chosen at random from the following
distribution. Color the nodes of $G$ with black and white like a chess
board.  White nodes are assigned binary values uniformly and
independently. Black nodes are assigned the label $+1$. 
%
%Partition $G$ into $3 \times 3$ node-disjoint subgrids.
%Center nodes of subgrids are assigned binary values uniformly and
%independently.  All other nodes are assigned the label 0.
%
\comment{
Given $\Y$, we generate the input at follows.  
For each
vertex $v \in V$, label it $-\Y(v)$ (a {\em bad node}) with probability $q 
\tfrac{1}{2}$ and $\Y(v)$ (a {\em good node}) with probability
$1-q$.%\footnote{That is, we are proving the lower bound in the more
%general model discussed in Section~\ref{ss:node}.  The basic model can
%be thought of as the $q=\tfrac{1}{2}$ case.}
We omit dependence on the parameter $q$ in what follows; it just needs
to be bounded away from~0.
For each edge $(u,v) \in E$, we flip a coin with probability $p <
\tfrac{1}{2}$.
If heads, the edge is {\em bad}; this means we label it ``$+1$'' if $\Y_u
\neq \Y_v$ and ``$-1$'' if $\Y_u = \Y_v$.  Otherwise, the edge is {\em good}
and we label it ``$+1$'' if $\Y_u = \Y_v$ and ``$-1$'' if $\Y_u \neq
\Y_v$.
}
Given $\Y$, input is generated using the random process described in Section~\ref{ss:ar}.

Consider an arbitrary function from inputs to labellings of $V$.  We
claim that the expected error of the output of this function,
where the expectation is over the choice of ground truth $\Y$ and the
subsequent random input, is $\Omega(p^2N)$.
%, assuming that $q$ is at
%least a fixed constant.  
This implies that, for every function, there
exists a choice of ground truth $\Y$ such that the expected error
of the function (over the random input) is $\Omega(p^2N)$.

Given $\Y$, call a white node {\em ambiguous} if exactly two of the
edges incident to it are labeled ``$+1$'' in the input. 
A white node is ambiguous with probability $6p^2(1-p)^2\ge 5.1p^2$ for
$p\le 0.078$.
% (whether or not the white node has label 0 or 1). 
Since there are $N/2$ white nodes, and the events corresponding to
ambiguous white nodes are independent, Chernoff bounds imply that 
there are at least $\tfrac{5p^2}{2}N$ ambiguous white nodes with very
high probability. 
%Given $X$, call a $3 \times 3$ subgrid {\em ambiguous} if exactly two
%of the edges incident to the center node are   labeled ``+'' in the
%input.  
%A subgrid is ambiguous with probability $6p^2(1-p)^2 \ge
%p^2$ (whether or not the center node has label 0 or 1).
%Since there are $N/9$ subgrids, and the events
%corresponding to ambiguous subgrids are mutually independent, Chernoff
%bounds imply that there are at least $\tfrac{p^2}{10}N$ ambiguous
%subgrids with very high probability.

Let $L$ denote the error contributed by ambiguous white nodes.  
Since the true labels of different white nodes are                             
conditionally independent (given that all black nodes are known                
to have value $+1$), the function that minimizes $\expect{L}$ just                
predicts each white node separately.
The algorithm that minimizes the expected value of $L$ 
simply predicts that each ambiguous white node has true label equal to
its input label.  This prediction is wrong with constant probability,
so $\expect{L} = \Omega(p^2N)$ for every algorithm.  Since $L$ is a
lower bound on the Hamming error, the result follows.

\section{Extensions \label{s:ext}}

The section sketches several extensions of our model and results, to
planar graphs beyond grids (Section~\ref{ss:gridlike}), 
%%to graphs with
%and 
to expander graphs (Section~\ref{ss:expanders}),
to graphs with a large minimum cut (Section~\ref{ss:bigcut}), and to
semi-random models (Section~\ref{ss:semi}).
%(Section~\ref{ss:semi}), and to node labels (Section~\ref{ss:node}).

\subsection{Approximate Recovery in Other Planar Graphs}\label{ss:gridlike}
%{\bf Other Planar Graphs:}

Section~\ref{s:grid} 
gives a polynomial-time
algorithm for essentially information-theoretically optimal approximate
recovery in grid graphs.  While the analysis 
does use properties of
grids beyond planarity, it is robust in that it applies to all planar
graphs that share two key features with grids.

The path graph (see Section~\ref{ss:ar}) shows that approximate recovery
is not possible for all planar graphs; additional conditions are
needed.  
%To state these conditions, we need to generalize the notion
%of a ``type 6'' region from Section~\ref{s:ub}.  
%Consider a graph $G=(V,E)$ embedded in the plane. 
%Let $\C$ denote the sets $S \sse V$ 
%such that the induced graph $G[S]$ is connected.  
%We assume that
%the boundary of $G$ contains at least 12 vertices, and we break it
%arbitrarily into {\em quarters}, each of which contains a contiguous
%subset of roughly 25\% of $G$'s boundary.
%We call $S$ {\em large} (corresponding to type 6) if it contains at
%least one node internal to each of the four quarters.
%We now define the filled-in version $F(S)$ of a small set $S \in \C$
%as follows.  Let $C_1,\ldots,C_k$ denote the connected components of
%$G[V \sm S]$.  Since $S$ is small, at least one of the $C_i$'s
%contains a quarter.  
%We define $F(S)$ as $S$ together with all of the $C_i$'s except
%for one that includes a quarter (if there is more than one, exclude an
%arbitrary one).  Thus, for every non-large set $S$, $F(S)$ excludes a
%quarter.
%Let $\F$ denote the set of filled-in regions.  As in
%Lemma~\ref{l:disjoint}, $F(S_1) = F(S_2)$ for $S_1,S_2 
%\in \F$ only if $S_1 \cap S_2 \neq \emptyset$.  
The first property, which fails in ``thin'' planar graphs
like a path but holds in many planar graphs of interest, is the
following weak expansion property:
\begin{itemize}

\item [(P1)] 
{\em (Weak expansion.)}  For some constants $c_1,c_2 > 0$,
  every filled-in set $F \in \F$ satisfies $|F| \le c_1|\delta(F)|^{c_2}$.

\end{itemize}
(Filled-in sets can be defined analogously to the grid case.)

The second key property is that 
the number of filled-in sets with a
given boundary size $i$ should be at most exponential in $i$.
% (at
%opposed to exponential in $i \ln N$).  
As in Lemma~\ref{l:count}, a
sufficient (but not necessary) condition for this property is that 
the
dual graph has bounded degree (except possibly for the vertex
corresponding to the outer face, which can have arbitrary degree).
\begin{itemize}

\item [(P2)] 
{\em (Bounded Dual Degree.)}  Every face of $G$, except
  possibly for the outer face, contains at most a constant~$c_3$
  number of edges.  
%Alternatively and weaker, there are constants
%  $c_3,c_4 > 0$ such that, for every $i$, the number of filled-in sets
%  $F \in \F$ with $|\delta(S)| = i$ is at most $c_3 \cdot N \cdot c_4^{i}$.
%
\end{itemize}

Our proof of computationally efficient
approximate recovery (Theorem~\ref{thm:ub})
% shows that every family of planar
%graphs  meeting~(P1) and~(P2) admits computationally efficient
%approximate recovery.
extends to show that approximate
recovery is possible in every planar graph that satisfies properties
(P1) and (P2); the precise bound on the function $f(p)$ depends on the
constants $c_1,c_2,c_3$.
%
%Since our recovery algorithm runs in polynomial time for every planar
%graph $G$ (by~\cite{barahona}), these graphs also permit
%computationally tractable approximate recovery.

\subsection{Approximate Recovery in Expander Graphs}
\label{ss:expanders}

%{\bf Expander Graphs:}
%
%QUESTION FOR AFTER DEADLINE: improve to p^{d/2} via perc theory?
%
%For a more involved example, 
%consider a family $\G$ of $d$-regular expanders.
Structured prediction on expander graphs is often applied to relational classification (e.g., predicting protein-protein  
interactions or web-page classification).  
This section proves 
that every family $\G$ of $d$-regular expanders
admits approximate recovery.
Recall the definition of such a family:
for some constant $c > 0$, 
for every $G \in \G$ with $N$ vertices and every set
$S \sse V$ with $|S| \le N/2$, $|\delta(S)| \ge c \cdot d
\cdot |S|$, where the boundary $\delta(S)$ is the set of edges with
exactly one endpoint in $S$.  
We claim that $\G$ allows approximate recovery with $f(p) = 3p/c$, and
proceed to the proof.

% TODO: change to be |S| \le N/d, and say that c \geq 1??
% TODO: change language from ``correlation clustering'' to ``cut''?

The algorithm is the same as in Section~\ref{s:grid}; it
is not computationally efficient for expanders. 
As in Section~\ref{s:grid}, analyzing the two-stage algorithm reduces
to analyzing the better of the two solutions produced by the first stage.
We therefore assume that the output $\hat{Y}$ of the first stage has
error $H$ at most $N/2$.
%The argument is constructive but not computationally efficient.

Fix a noise parameter $p \in (0,\tfrac{1}{2})$, a
graph $G \in \G$ with $N$ sufficiently large,
and a ground truth.
% $\Y:V \rightarrow \{0,1\}$.
%Given an input, our algorithm computes the optimal correlation
%clustering with at
%most 2 clusters --- the partition of $G$ that maximizes the number of
%``-'' edges that are cut plus the number of ``+'' edges that are not
%cut.  As noted above, of the two corresponding labellings, we consider
%the one $\Yhat$ with smaller error (which is then at most $N/2$). 
Let $B$
denote the set of bad edges.
% (i.e., inconsistent with the ground truth) and
%$H$ the Hamming error of the algorithm. 
%
%We begin with a simple consequence of concentration
%inequalities (e.g., Chernoff bounds): 
Chernoff bounds imply that for all sufficiently large $N$,
the probability that $|B| \ge 2p|E| = pdN$ is at most $1/N^2$.
When $|B| > pdN$, we can trivially bound the error $H$ by $N/2$.
When $|B| \le pdN$, we bound $H$ from above as follows.

Let $S$ denote the nodes of $V$ correctly
classified by the first stage $\Yhat$ and $C_1,\ldots,C_k$ the
connected components of 
the (misclassified)
nodes of the induced subgraph $G[V \sm S]$.  
Since $H/2$, $|C_i| \le N/2$ for every $i$.
We have
$$
H = \sum_{i=1}^k |C_i| 
 \le  \frac{1}{cd} \cdot \sum_{i=1}^k |\delta(C_i)|\\
 \le  \frac{2}{cd} \cdot \sum_{i=1}^k |\delta(C_i) \cap B|\\
 \le  \frac{2}{cd} \cdot |B|,
% \le  \frac{c_2(\delta)}{c} \cdot p \cdot N,
$$
where the first inequality follows from the expansion condition, the
second from the Flipping Lemma (Lemma~\ref{l:flipping}),
%fact that half the edges of each $\delta(C_i)$ are bad, 
and the third from the fact that the $\delta(C_i)$'s are disjoint
(since the $C_i$'s are maximal).  Thus, when $|B| \le pdN$, $H \le
\tfrac{2p}{c} N$. Overall, we have 
$$
\expect{H} \le 1 \cdot \frac{2p}{c}N + \frac{1}{N^2} \cdot \frac{N}{2}
\le \frac{3p}{c}N
$$
%$\expect{H} \le 3pN/c$
for $N$ sufficiently large, as claimed.

\subsection{Graphs with a Large Min Cut}\label{ss:bigcut}

Approximate recovery is also possible in every graph family
$\G$ for which the global minimum cut $c^*$ is bounded below
by $c \log N$ for a sufficiently large constant~$c$.
%vertices $N$
%is $\Omega(\log N)$, for a
%sufficiently 
%large constant $c_1$.  Note t
This class of graphs is incomparable to the
expanders considered in Section~\ref{ss:expanders}.

To see why a large minimum cut is sufficient, we modify the 
first-stage analysis in the proof of
Theorem~\ref{thm:ub} as follows.  Define $\C$ as the subsets $S$ of $V$
such that $|S| \le 
N/2$ and $G[S]$ is connected, and $\C_i$ the subset of $\C$
corresponding to sets $S$ with $|\delta(S)|=i$.  Recall that, for
every $\alpha \ge 1$, the number of $\alpha$-approximate minimum cuts
of an undirected graph is at most $N^{2\alpha}$ (e.g., see \cite{karger1993global}).  Thus, $|\C_i| \le N^{2i/c^*}$, which is
at most $2^{2i/c}$ when $c^* \ge c \log_2 N$.  That is, there
can only be an exponential number of connected subgraphs with a given
boundary size (cf., property (P1) in Section~\ref{ss:gridlike}).  A
calculation along the lines of the proof of Theorem~\ref{thm:ub}, 
then implies that approximate recovery is
possible, provided the constant $c$ is sufficiently large.

\subsection{Semi-Random Models}\label{ss:semi}

All of our positive results make minimal use of the properties of the
random process that generates inputs given the ground truth.
Our proofs only need the fact that the probability that a boundary
$\delta(S)$ consists of at least half bad edges decays exponentially
in the boundary size $|\delta(S)|$ (Lemma~\ref{l:chernoff}).  As such,
our positive results are robust to many variations in the random
model.

For example, the fact that every edge has the same noise parameter $p$
is not important --- our algorithms continue to have the exact same
guarantees, with the same proofs, with the function $f(p)$ replaced by
$f(p_{\max})$, where $p_{\max}$ is the maximum noise parameter of any
  edge.  If bad edges are negatively correlated instead of
  independent, then the relevant  Chernoff bounds (and hence
  Lemma~\ref{l:chernoff}) continue to hold (see
  e.g. \cite{panconesibook}), and our results remain unchanged.

Most interestingly, our positive results can accommodate the following
semi-random adversary (cf., \cite{feigekillian}).  Given a graph $G$
and ground truth, as before nature independently designates each
edge as good or bad with probability $1-p$ and $p$, and
similarly for nodes (with probability $1-\q$ and $\q$).  Good nodes
and edges are labelled according to the ground truth.  An adversary,
who knows what algorithm will be used on the input, selects arbitrary
labels for the bad nodes and edges.  Our basic models corresponds to
the special 
case in which the adversary labels every bad node and edge to be
inconsistent with the input.  Such semi-random adversaries can often
foil algorithms that work well in a purely random model, especially
algorithms that are overly reliant on the details of the input
distribution or that are ``local'' in nature.

In all of our proofs of our positive results, we effectively assume
that every relevant set $S$ that has a boundary $\delta(S)$ with at
least half bad edges contributes $|S|$ to our algorithm's error.  
Thus, an adversary maximizes our error upper bound by maximizing the
number of bad nodes and edges.
%such sets, which can be done by simply maximizing the number
%of bad edges. 
In other words, from the standpoint of our error
bounds, a semi-random adversary is no worse than a random one.

\section{Empirical Study}
\label{sec:empirical}
Our theoretical analysis suggests that statistical recovery on 2D grid graphs can attain an error that scales with $p^2$. Furthermore, we show that this error is achieved using the two-step algorithm in \secref{s:grid}.
% which first ignores the node evidence, finds the optimal assignment given the edge evidence up to symmetry, and then breaks the symmetry using node evidence. Since we also give a lower bound of $p^2$ on the Hamming error, this implies that the two step algorithm is close to optimal for low edge noise (since this is the regime for which we obtain the result). 
Here we describe a synthetic experiment that compares the two-step algorithm to other recovery procedures. We consider a $20\times 20$ grid, with high node noise of $0.4$ and variable edge noise levels. In addition to the two-step algorithm we consider the following:\footnote{We also experimented with a greedy hill climbing procedure, but results were poor and are not shown.} 
\begin{list}{\labelitemi}{\leftmargin=1em}
\item Marginal inference - predicting according to $p(\Y_i|\X)$. As mentioned in \secref{ss:ar} this is the optimal recovery procedure. Although it is generally hard to calculate, for the graph size we use it can be done in $20$ minutes per model.
%\item Our two step procedure.
\item Local LP relaxation - Instead of calculating $p(\Y_i|\X)$ one can resort to approximation. One possibility is to calculate the mode of $p(\Y|\X)$ (also known as the MAP problem). However, since this is also hard, we consider LP relaxations of the MAP problem. The simplest such relaxation assumes locally consistent pseudo-marginals.
% (e.g., see \cite{sontag2007new}).
\item Cycle LP relaxation - A tighter version of the LPs above uses cycle constraints instead of pairwise. In fact, for planar graphs with no external field (as in the first step of our two step algorithm) this relaxation is tight. It is thus of interest to study it in our context. For both the cycle and local relaxations we use the code for \cite{SontagChoeLi_uai12}.
\end{list}

Fig.~\ref{fig:2020} shows the expected error for the different algorithms, as a function of edge noise. It can be seen that the two step procedure almost matches the accuracy of the optimal marginal algorithm for low noise levels. As the noise increases the gap grows. Another interesting observation is that the local relaxation performs significantly worse than the other baselines, but that the cycle relaxation is close to optimal. The latter observation is likely to be due to the fact that with high node noise and low edge noise, the MAP problem is ``close'' to the no node-noise case, where the cycle relaxation is exact. However, an analysis of the Hamming error in this case remains an open problem.

\begin{figure}[!h]
\label{fig:2020}
\centering
\vspace{-1mm}
\includegraphics[width=0.4\textwidth]{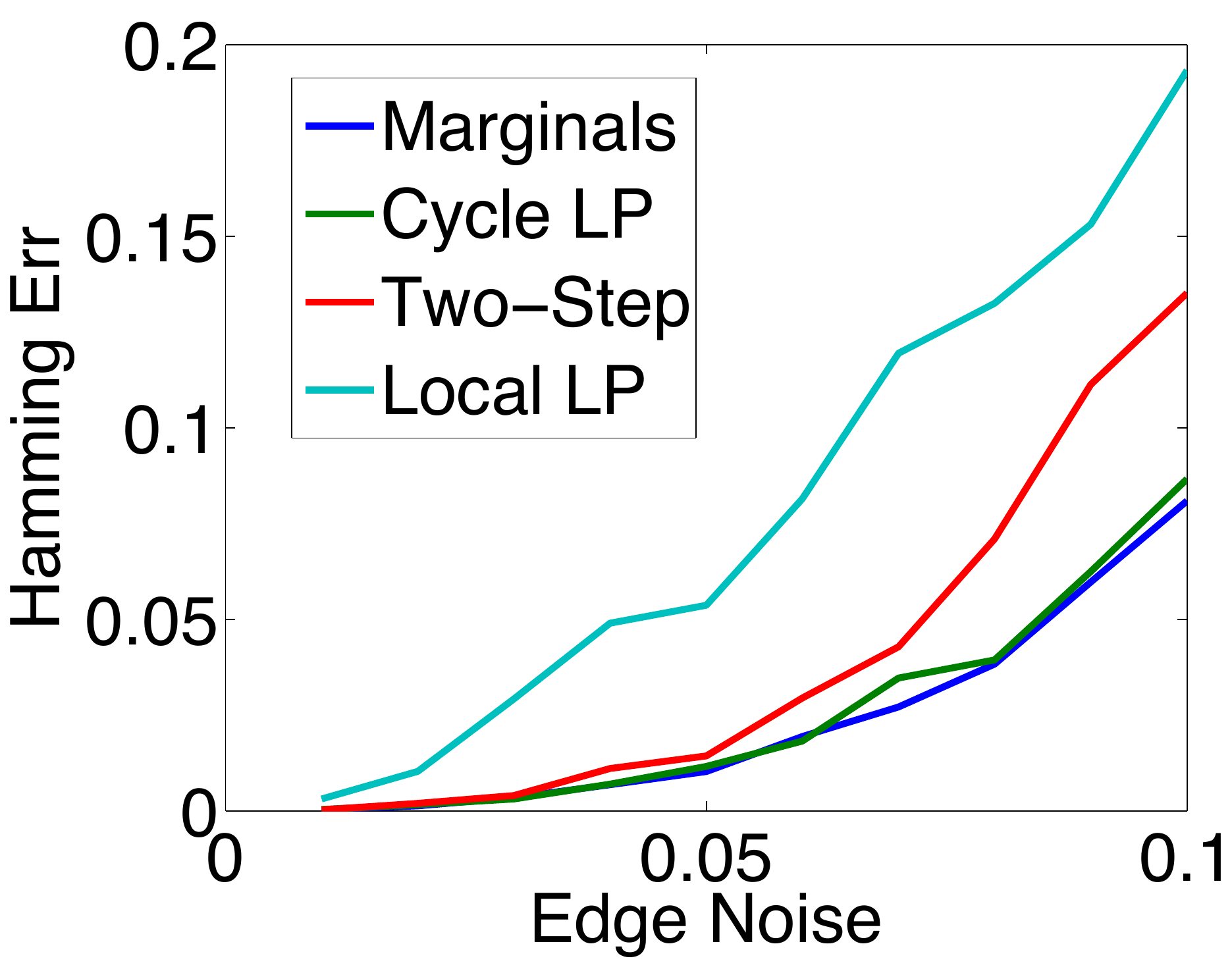}\vspace{-4mm}
\caption{Average Hamming error for different recovery algorithms. Data is generated from a $20\times20$ grid with node noise $q=0.4$ and variable edge noise $p$. The true $\Y$ is the all zeros word. Results are averaged over $100$ repetitions.}
\end{figure}

\ignore{
should result in Hamming error on the order of $q^2$ for
Show empirically that in the regime of large node noise ($q$) , the 2-step algorithm gets close to information
theoretically optimal (marginal inference). Show in regime of large q, 2 step algorithm gets close to information
theoretically optimal. ICM fails. Also show results for c-MAP, to point out that it is quite close to marginal inference.
}

\section{Discussion}
Structured prediction underlies many empirically successful systems in
machine vision and NLP.  In most of these (e.g.,
\cite{koo2010dual,Kappes13}) the inference problems are intractable
and approximate inference is used instead. However, there is little
theoretical understanding of when structured prediction is expected to
perform well, how its performance is related to the structure of the
score function, which approximation algorithms are expected to work in
which setting, etc. 

In this work we present a first step in this direction, by analyzing
the error of structured prediction for 2D grid models. One key finding
is that a two-step algorithm attains the information theoretically
optimal error in a certain regime of parameters. What makes this
setting particularly interesting from a theoretical perspective is
that exact inference (marginals and MAP) is intractable due to the
intractability of planar models with external fields. Thus, it is
rather surprising that a tractable algorithm achieves optimal
performance.

%For instance, are there cases where greedy algorithms such as alpha expansion have recovery guarantees.  

%Our proofs only use the fact that the probability that a boundary 
%$\delta(S)$ consists of at least half bad edges decays exponentially 
%in the boundary size $|\delta(S)|$ (Lemma~\ref{l:chernoff}).  As such,
%our positive results are robust to many variations in the random 
%model.  For example, the fact that every edge has the same noise parameter $p$
%is not important, and a simple variation on the analysis yields an
%upper bound of $O(Np_{max}^2)$ where $p_{max}$ is the maximum edge
%noise parameter.   

Our work opens the door to a number of new directions, with both
theoretical and practical implications. In the context of grid models,
we have not studied the effect of the node noise $q$ but rather
assumed it may arbitrary (less than $0.5$). Our two step procedure uses
both node and edge evidence, but it is clear that for small $q$,
improved procedures are available. In particular, the experiments in
\secref{sec:empirical} show that decoding with cycle LP relaxations
results in empirical performance that is close to optimal. More
generally, we would like to understand the statistical and
computational properties of structured prediction for complex tasks
such as dependency parsing \cite{koo2010dual} and non-binary variables
(as in semantic segmentation). In these cases, it would be interesting
to understand how the structure of the score function affects both the
optimal expected accuracy and the algorithms that achieve it.

\ignore{
Re-iterate the high-level questions: {\em information theoretic}: is
approximate recovery possible? {\em computational}: is efficient approximate
recovery possible? These high-level questions are non-trivial and
relevant for many different problems. 
}

\ignore{
Our work opens the door to a number of new directions, with both
theoretical and practical implications.
}
\ignore{
Our proofs only need the fact that the probability that a boundary 
$\delta(S)$ consists of at least half bad edges decays exponentially 
in the boundary size $|\delta(S)|$ (Lemma~\ref{l:chernoff}).  As such,
our positive results are robust to many variations in the random 
model. 
For example, the fact that every edge has the same noise parameter $p$
is not important --- our algorithms continue to have the exact same 
guarantees, with the same proofs, with the function $f(p)$ replaced by 
$f(p_{\max})$, where $p_{\max}$ is the maximum noise parameter of any 
  edge.  If bad edges are negatively correlated instead of 
  independent, then the relevant  Chernoff bounds (and hence 
  Lemma~\ref{l:chernoff} continue to hold, and our results remain unchanged.
  }
%see   e.g. \cite{panconesibook})

%\input{appendix}

\appendix

%\paragraph{Marginal Inference is Minimax Optimal Algorithm}
\section{Marginal Inference is the Minimax Optimal Algorithm}
\label{ss:MI_minimax}

In this section, we prove marginal inference using the uniform
prior, which we denote by ${\algo}_1$, is the minimax optimal
algorithm (i.e., minimizes $e({\algo}) = \max_\y e_\y(\algo)$).
The marginal inference algorithm predicts each node separately by $\Yhat_i
\leftarrow \arg\max_{\Y_i} p(\Y_i \mid X)$ using the uniform prior
over $X$.
%The key property of this algorithm is that it is symmetric, i.e.,
%$e_\y({\algo}_1)$ is the same for all ground truth assignments $\y$.

Assume for contradiction that there is an algorithm ${\algo}_0$ that
yields strictly smaller error than marginal inference. Hence, by
definition the of a minimax optimal algorithm, 
there exists ground truth assignments $\y_0$ and $\y_1$ such that
$\max_\y 
e_\y(\algo _0) = e_{\y_0}({\algo}_0) < e_{\y_1}({\algo}_1)$. 
By symmetry, the marginal inference algorithm has equal error for
every ground truth.
Hence, $e_\y({\algo}_0) < e_\y({\algo}_1)$ for every
ground truth assignments $\y$.

On the other hand, marginal inference minimizes the expected 
Hamming error when the prior distribution over ground truth assignments
is uniform.
%, i.e., it minimizes 
%$\expect [\y] {e_\y({\algo})}$ where the expectation is over uniform 
%distribution of ground truth assignments. 
To see why, let $\Yhat _i
(\X)$ be an estimator of the $i$-th node. The expected Hamming error
of $\Yhat _i$ assuming uniform prior on $\Y_i$ is
\be
  \prob{\Yhat_i(\X) \ne \Y _i} = \sum _\X 1 _{\Yhat_i(\X)=1} \cdot
  \frac{1}{2} \prob{\mu _i ^{-}(\X)} + \sum _\X 1 _{\Yhat_i(\X)=-1}
  \cdot \frac{1}{2} \prob{\mu _i ^{+}(\X)} 
\ee 
where $\mu _i ^{+}$ and $\mu _i ^{-}$ are the distributions of $\X$
conditioned on $\Y_i=+1$ and $\Y_i=-1$, respectively. 
%It is easy to
%see that marginal inference using uniform prior minimizes this
%quantity. Note that
Since $\expect [\y] {e_\y({\algo})}$ is the sum of the
expected error at individual nodes, marginal inference using
the uniform prior minimizes it.
%$\expect [\y] {e_\y({\algo})}$ since it
%%minimizes individual summands.  
The optimality of marginal inference for the uniform prior
contradicts the fact that ${\algo}_0$
performs better than marginal inference on all ground truth
assignments. 
%In summary, marginal inference has the same error for any
%ground truth assignment and it is optimal for uniform prior, hence the
%proof follows from standard statistical decision theory
%arguments. (See \cite{berger}) 

Notice that this proof also works for the subset of ground truths
considered in the proof of lower bound for the grids
(Section~\ref{sec:lb}).

\section{Illustration of Filled In Sets}
\label{sec:illustration_filled_in_sets}

Recall that for every subset $S$ we defined a corresponding filled in
set $F(S)$. Figures~\ref{fig:type1}--\ref{fig:type5} illustrate the
transformation from a subset $S$ to the corresponding filled-in set
$F(S)$.  

\begin{figure}[!h]
\label{fig:type1}
\centering
\includegraphics[width=\textwidth]{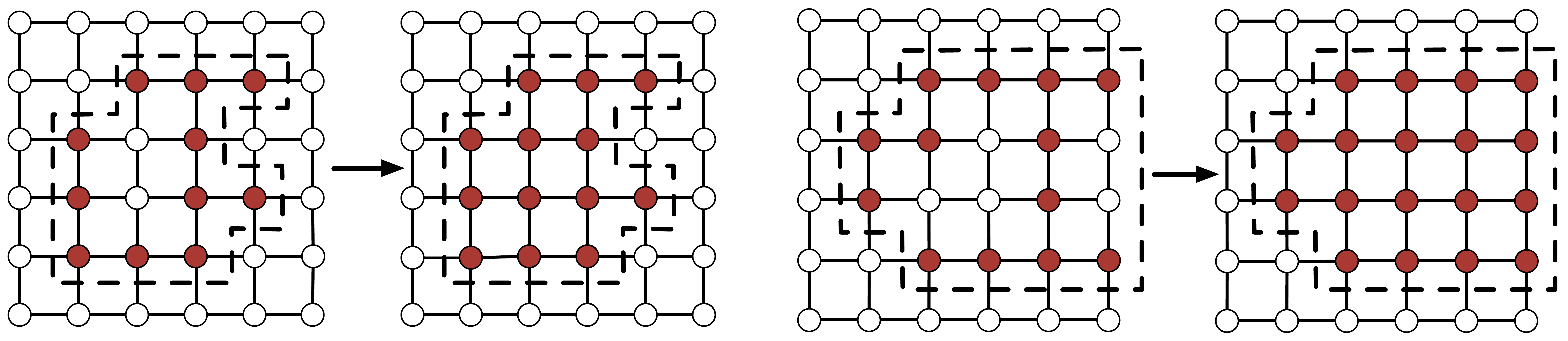}
\caption{An example of a type 1 set (left) and a type 2 set (right)
and the corresponding filled-in sets.}
\end{figure}

\begin{figure}[!h]
\label{fig:type2}
\centering
\includegraphics[width=\textwidth]{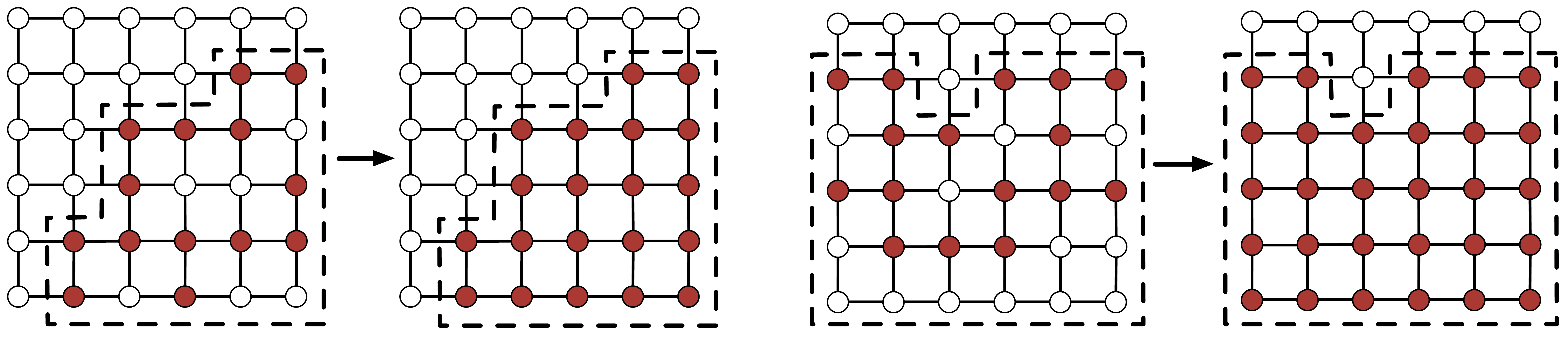}
\caption{An example of a type 3 set (left) and a type 4 set (right)
and the corresponding filled-in sets.}
\end{figure}

\begin{figure}[!h]
\label{fig:type5}
\centering
\includegraphics[width=105mm]{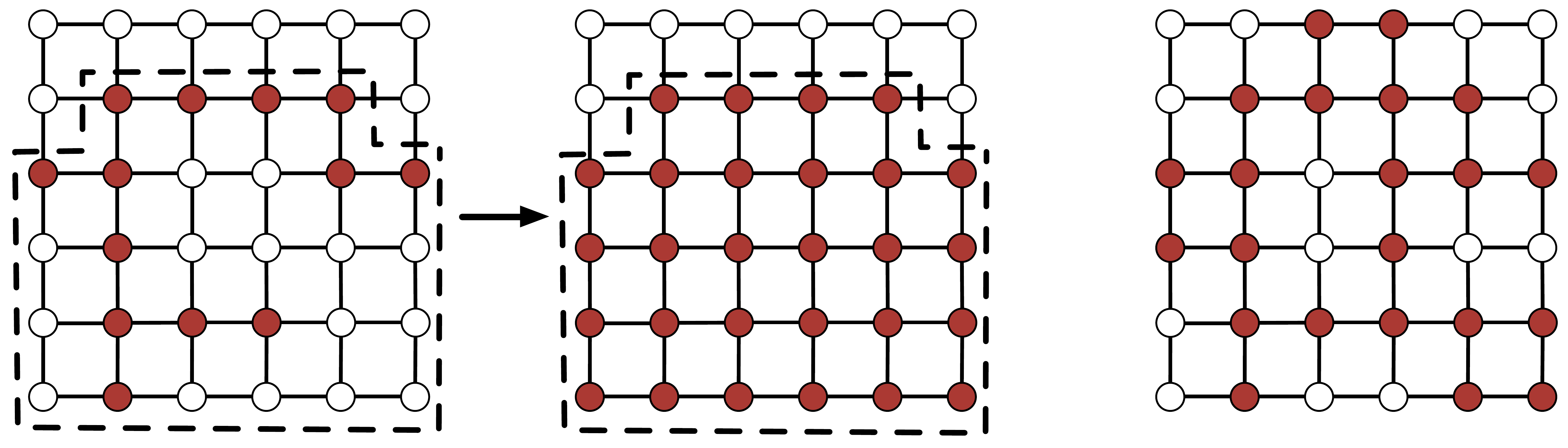}
\caption{An example of a type 5 set and the corresponding filled-in
  set (left) and an example of type 6 set.} 
\end{figure}

\end{document}